\relax
\documentclass[twoside]{article}
\usepackage[accepted]{aistats2018}

\usepackage{helvet}  
\usepackage{courier}  
\usepackage{url}  

\newcommand\blfootnote[1]{%
  \begingroup
  \renewcommand\thefootnote{}\footnote{#1}%
  \addtocounter{footnote}{-1}%
  \endgroup
}

\usepackage{booktabs,ragged2e,amsfonts,tabularx,hyperref,url,courier,csquotes,dblfloatfix,setspace,ifthen,amsmath, amsthm, amssymb, graphicx,amsfonts,multirow,subfig,algorithmic,algorithm,hyperref}
\usepackage{color}
\usepackage{authblk}

 \newtheorem{theorem}{Theorem}
\newtheorem{lemma}[theorem]{Lemma}
\begin{document}

\twocolumn[

\aistatstitle{Robust Locally-Linear Controllable Embedding}

\aistatsauthor{Ershad Banijamali$^1$ \And Rui Shu$^2$ \And Mohammad Ghavamzadeh$^3$ \And Hung Bui$^3$ \And Ali Ghodsi$^1$}

\aistatsaddress{ \\$^1$University of Waterloo $\qquad\qquad$ $^2$Standford University $\qquad\qquad$ $^3$DeepMind } 
]

\begin{abstract}
Embed-to-control (E2C)~\cite{watter2015embed} is a model for solving high-dimensional optimal control problems by combining variational auto-encoders with locally-optimal controllers. However, the E2C model suffers from two major drawbacks: {\bf 1)} its objective function does not correspond to the likelihood of the data sequence and {\bf 2)} the variational encoder used for embedding typically has large variational approximation error, especially when there is noise in the system dynamics. In this paper, we present a new model for learning {\em robust} locally-linear controllable embedding (RCE). Our model directly estimates the predictive conditional density of the future observation given the current one, while introducing the bottleneck~\cite{shu2016bottleneck} between the current and future observations. Although the bottleneck provides a natural embedding candidate for control, our RCE model introduces additional specific structures in the generative graphical model so that the model dynamics can be robustly linearized. We also propose a principled variational approximation of the embedding posterior that takes the future observation into account, and thus, makes the variational approximation more robust against the noise. Experimental results show that RCE outperforms the E2C model, and does so significantly when the underlying dynamics is noisy.
\end{abstract}


\section{Introduction}
\label{sec:intro}

Model-based locally optimal control algorithms are popular in controlling non-linear dynamical systems with continuous state and action spaces. Algorithms from this class such as differential dynamic programming (DDP)~\cite{Jacobson70DD}, iterative linear quadratic regulator (iLQR)~\cite{li2004iterative}, and iterative linear quadratic Gaussian (iLQG)~\cite{Todorov05GI} have been successfully applied to a variety of complex control problems~\cite{Atkeson02NP,Tassa08RH,Levine13VP,Pan14PD}. The general idea of these methods is to iteratively linearize the non-linear dynamics around the current trajectory and then use linear quadratic methodology to derive Riccati-like equations to improve the trajectory. However, these methods assume that the model of the system is known and need relatively low-dimensional state representations. These requirements limit their usage in control of dynamical systems from raw sensory data (e.g.,~image and audio), a scenario often seen in modern reinforcement learning (RL) systems.    


Although both model-based RL and methods to find low-dimensional representations that are appropriate for control (see e.g.,~\cite{bohmer2015autonomous}) have a long history, they have recently witnessed major improvements due to the advances in the field of deep learning. Deep autoencoders~\cite{lange2010deep,wahlstrom2015pixels} have been used to obtain low-dimensional representations for control, and deep generative models have been used to develop new model-based RL algorithms. However, what is desirable in model-based locally optimal control algorithms is a representation that can be used for learning a model of the dynamical system and can also be systematically incorporated into the existing tools for planning and control. One such model is embed to control (E2C)~\cite{watter2015embed}. E2C turns the problem of locally optimal control in high-dimensional non-linear systems into one of identifying a low-dimensional latent space in which we can easily perform locally optimal control. The low-dimensional latent space is learned using a model based on variational autoencoders (VAEs)~\cite{vae2014,rezende2014stochastic} and the iLQG algorithm~\cite{Todorov05GI} is used for locally optimal control.\blfootnote{Part of the work was done when the first four authors were at Adobe Research.}



While the idea of E2C is intriguing, it suffers from two major statistical deficiencies. {\bf Firstly}, to induce the lower-dimensional embedding, at each time step $t$, E2C models the pair-marginal distribution of two adjacent observations $(\mathbf{x}_t, \mathbf{x}_{t+1})$. As a result, its loss function effectively is the sum over the pair-marginals, which is clearly not the data likelihood for the entire trajectory. Moreover, at every time step $t$, E2C needs to enforce the consistency between the posterior of the embedding and the predictive distribution of the future embedding by minimizing their KL divergence. These all indicate that the E2C loss is not a lower-bound of the likelihood of the data. The practice of modeling the pair-marginal of $(\mathbf{x}_t, \mathbf{x}_{t+1})$ using a latent variable model also imposes a Gaussian prior on the embedding space, which might be in conflict with the locally-linear constraint that we would like to impose. {\bf Secondly}, the variational inference scheme in E2C attempts to approximate the posterior of the latent embedding via a recognition model that does not depend on the future observation $\mathbf{x}_{t+1}$. We believe that this is done out of necessity, so that the locally-linear dynamics can be encoded as a constraint in the original E2C model. In an environment where the future is uncertain (e.g.,~in the presence of noise or other unknown factors), the future observation carries significant information about the true posterior of the latent embedding. Thus, a variational approximation family that does not take future observation into account, while approximating the posterior, will result in a large variational approximation error, leading to the learning of a sub-optimal model that underperforms, especially when the dynamics is noisy.

To address these issues, we take a more systematic view of the problem. Instead of mechanically applying VAE to model the pair-marginal density, we build on the recent bottleneck conditional density estimator (BCDE)~\cite{shu2016bottleneck} and directly model the predictive conditional density $p(\mathbf{x}_{t+1}\vert \mathbf{x}_{t})$. The BCDE model introduces a bottleneck random variable $\mathbf{z}_t$ in the middle of the information flow from $\mathbf{x}_{t}$ to $\mathbf{x}_{t+1}$. While this bottleneck provides a natural embedding candidate for control, these embeddings need to be structured in a way to respect the locally linear constraint of the dynamics. Our proposed model, {\em robust controllable embedding} (RCE), provides a rigorous answer to this question in the form of a generative graphical model. A key idea is to explicitly treat the reference linearization point in the locally-linear model as an additional random variable. We also propose a principled variational approximation of the embedding posterior that takes the future observation into account and optimizes a variational lower-bound of the likelihood of the data sequence. This allows our framework to provide a clean separation of the generative graphical model and the amortized variational inference mechanism (e.g.,~the recognition model). 

After a brief overview of locally-linear control and E2C in Section~\ref{sec:prelim}, we present our proposed RCE model in Section~\ref{sec:RCE}. Unlike E2C, RCE directly models the conditional density of the next observation given the current one via a form of bottleneck conditional density estimators~\cite{shu2016bottleneck}. In Section~\ref{sec:RCE}, we first describe the RCE's graphical model in details and then present the proposed variational approximation of the embedding's posterior. In Section~\ref{sec:experiments}, we apply RCE to four RL benchmarks from~\cite{watter2015embed} and show that it consistently outperforms E2C in both prediction and planning. Crucially, we demonstrate the robustness of RCE, i.e.,~as we add noise to the dynamics, RCE continues to perform reasonably well while E2C's performance degrades sharply. 


\section{Preliminaries}
\label{sec:prelim}

In this section, we first define the non-linear control problem that we are interested to solve, and then provide a brief overview of stochastic locally optimal control and the E2C model. We also motivate our proposed robust controllable embedding (RCE) model that will be presented in Section~\ref{sec:RCE}. 


\subsection{Problem Formulation}
\label{subsec:formulation}

We are interested in controlling the non-linear dynamical systems of the form
\begin{equation}
\label{eq:true_dynamic}
\mathbf{s}_{t+1} = f_{\mathcal{S}}(\mathbf{s}_t,\mathbf{u}_t) + \mathbf{n}^{\mathcal{S}},
\end{equation}
where $\mathbf{s}_t\in\mathbb{R}^{n_s}$ and $\mathbf{u}_t\in\mathbb{R}^{n_u}$ denote the state and action of the system at time step $t$, $\mathbf{n}^{\mathcal{S}}\sim\mathcal{N}(\mathbf{0},\mathbf{\Sigma}_{\mathbf{n^{\mathcal{S}}}})$ is the Gaussian system noise, and $f_{\mathcal{S}}$ is a smooth non-linear system dynamics. Note that in this case $p(\mathbf{s}_{t+1}|\mathbf{s}_t,\mathbf{u}_t)$ would be the multivariate Gaussian distribution $\mathcal{N}\big(f_{\mathcal{S}}(\mathbf{s}_t,\mathbf{u}_t),\mathbf{\Sigma}_{\mathbf{n^{\mathcal{S}}}}\big)$. We assume that we only have access to the high-dimensional observation $\mathbf{x}_t\in\mathbb{R}^{n_x}$ of each state $\mathbf{s}_t$ ($n_s \ll n_x$) and our goal is to learn a low-dimensional latent state space $\mathcal{Z}\subset \mathbb{R}^{n_z}$ ($n_z\ll n_x$) in which we perform optimal control. 



\subsection{Stochastic Locally Optimal Control}
\label{subsec:LLC}

Stochastic locally optimal control (SLOC) is based on the idea of controlling the non-linear system~\eqref{eq:true_dynamic}, along a reference trajectory $\{\bar{\mathbf{s}}_1,\bar{\mathbf{u}}_1,\ldots,\bar{\mathbf{s}}_H,\bar{\mathbf{u}}_H,\bar{\mathbf{s}}_{H+1}\}$, by transforming it to a time-varying linear quadratic regulator (LQR) problem
\begin{align}
\min_{\mathbf{u}_{1:T}}\;&\mathbb{E}\left[\sum_{t=1}^T\big((\mathbf{s}_t-\mathbf{s}^f)^\top\mathbf{Q}(\mathbf{s}_t-\mathbf{s}^f) + \mathbf{u}_t^\top\mathbf{R}\mathbf{u}_t\big)\right] \nonumber \\
&\text{s.t} \quad \mathbf{y}_{t+1} = \mathbf{A}_t\mathbf{y}_t+\mathbf{B}_t\mathbf{v}_t,
\label{eq:SLOC}
\end{align}
%
where $\mathbf{s}^f$ is the final (goal) state, $\mathbf{Q}$ and $\mathbf{R}$ are cost weighting matrices, $\mathbf{y}_t=\mathbf{s}_t-\bar{\mathbf{s}}_t$, $\mathbf{v}_t=\mathbf{u}_t-\bar{\mathbf{u}}_t$, $\bar{\mathbf{s}}_{t+1}=f_{\mathcal{S}}(\bar{\mathbf{s}}_t,\bar{\mathbf{u}}_t)$, $\mathbf{A}_t=\frac{\partial f_{\mathcal{S}}}{\partial\mathbf{s}}(\bar{\mathbf{s}}_t,\bar{\mathbf{u}}_t)$, and $\mathbf{B}_t=\frac{\partial f_{\mathcal{S}}}{\partial\mathbf{u}}(\bar{\mathbf{s}}_t,\bar{\mathbf{u}}_t)$. Eq.~\ref{eq:SLOC} indicates that at each time step $t$, the non-linear system has been locally approximated with a linear system around the reference point $(\bar{\mathbf{s}}_t,\bar{\mathbf{u}}_t)$ as 
%
\begin{align}
\label{eq:local-lqr}
\mathbf{s}_{t+1} \approx f_{\mathcal{S}}(\bar{\mathbf{s}}_t,\bar{\mathbf{u}}_t) &+ \left[\frac{\partial f_{\mathcal{S}}}{\partial\mathbf{s}}(\bar{\mathbf{s}}_t,\bar{\mathbf{u}}_t)\right](\mathbf{s}_t - \bar{\mathbf{s}}_t) \\ 
&+ \left[\frac{\partial f_{\mathcal{S}}}{\partial\mathbf{u}}(\bar{\mathbf{s}}_t,\bar{\mathbf{u}}_t)\right](\mathbf{u}_t - \bar{\mathbf{u}}_t). \nonumber
\end{align}
The RHS of Eq.~\ref{eq:SLOC} sometimes contains an offset $\mathbf{c}_t$ resulted from the linear approximation and/or noise   
\begin{equation}
\label{eq:offset}
\mathbf{y}_{t+1} = \mathbf{A}_t\mathbf{y}_t+\mathbf{B}_t\mathbf{v}_t+\mathbf{c}_t. 
\end{equation}
Eq.~\ref{eq:offset} can be seen as 
\begin{equation*}
\begin{bmatrix}
\mathbf{y}_{t+1} \\
1 
\end{bmatrix} = 
\begin{bmatrix}
\mathbf{A}_t  & \mathbf{c}_t \\
0  & 1 
\end{bmatrix}
\begin{bmatrix}
\mathbf{y}_t \\
1 
\end{bmatrix} +
\begin{bmatrix}
\mathbf{B}_t \\
0 
\end{bmatrix} 
\mathbf{v}_t,
\end{equation*}
and thus, can be easily transformed to the standard form~\eqref{eq:SLOC} by adding an extra dimension to the state as 
\begin{equation*}
\mathbf{y}'_t  = \begin{bmatrix}    \mathbf{y}_{t} \\1 \end{bmatrix}, \hspace{.5cm} 
\mathbf{A}'_t  =  \begin{bmatrix} \mathbf{A}_t  & \mathbf{c}_t\\  0  & 1 \end{bmatrix}, \hspace{.5cm}
\mathbf{B}'_t  =  \begin{bmatrix} \mathbf{B}_t  \\  0 \end{bmatrix}.
\end{equation*}
Locally optimal actions in Eq.~\ref{eq:SLOC} can be computed in closed-form by solving the local LQRs~\eqref{eq:local-lqr} using the value iteration algorithm. 

Since the quality of the control depends on the quality of the reference trajectory, SLOC algorithms are usually iterative (e.g.,~iLQR and iLQG), and at each iteration generate a better reference trajectory. At the abstract level, a SLOC algorithm operates as follows: at each iteration $k$, a reference trajectory is generated using the current policy $\pi^{(k)}$, the LQR approximation of the non-linear system is computed around this reference trajectory, and finally the next policy $\pi^{(k+1)}$ is computed by solving this LQR. The algorithm stops after a fixed number of iterations, e.g.,~$100$. 

As mentioned in Section~\ref{subsec:formulation}, since we do not have access to the true state $\mathbf{s}$, we perform the optimal control in the low-dimensional latent space $\mathbf{z}$ learned from the observations $\mathbf{x}$. Thus, all the $\mathbf{s}$'s in this section should be replaced by $\mathbf{z}$ in the following sections.

\subsection{Embed to Control (E2C) Model}
\label{subsec:E2C}


We now return to the assumption that we only observe a finite number of high-dimensional sensory data (e.g.,~images) $\mathbf{x}_t\in\mathbb{R}^{n_x}$ from the system. We denote the high-dimensional observation sequence by $\mathbf{X} = \{\mathbf{x}_1,\mathbf{x}_2,...,\mathbf{x}_N\}$. Note that the observations are selected such that the sequence $\mathbf{X}$ is Markovian. For example, $\mathbf{x}$ could be a set of buffered observed images of the system that encodes all the information about the past. Depending on the system, this set may have only one or multiple images. 

It is clear that direct control in $\mathbb{R}^{n_x}$ is complicated because of its high-dimensional nature. However, when the true underlying state space is low-dimensional, it would be possible to embed the high-dimensional observations in a low-dimensional latent space $\mathcal{Z}$, in a way that the dynamics of the system can be captured by a much simpler model, which can then be used for optimal control. This general strategy is known as embed to control (E2C)~\cite{watter2015embed}. Note that a suitable embedding function is sufficient for model-based control, we do not need to recover the true state $\mathbf{s}_t$. 

We denote by $\mathbf{z}_t$ the low-dimensional embedding of $\mathbf{x}_t$. E2C first introduces a new variable $\hat{\mathbf{z}}_{t+1}$ as the result of applying $\mathbf{u}_t$ to the latent dynamics $f_\mathcal{Z}$, i.e.,
\begin{equation}
\hat{\mathbf{z}}_{t+1} = f_{\mathcal{Z}}(\mathbf{z}_t,\mathbf{u}_t) + \mathbf{n}^{\mathcal{Z}}_t,
\label{eq:e2c_dynamic}
\end{equation}
where $\mathbf{n}^{\mathcal{Z}}_t$ denotes the transition noise in the latent space. E2C employs the pair $(\mathbf{z}_t,\hat{\mathbf{z}}_{t+1})$ as the latent variables that model the pair-marginal $p(\mathbf{x}_t,\mathbf{x}_{t+1})$. It uses the variational recognition network $q(\mathbf{z}_t\vert\mathbf{x}_t)$, while forcing $q(\mathbf{\hat{z}}_{t+1} | \mathbf{z}_t, \mathbf{u}_t)$ to be the generative dynamics of Eq.~\ref{eq:e2c_dynamic}. This leads to the following lower-bound of the pair-marginal
\begin{align}
&\log p(\mathbf{x}_t, \mathbf{x}_{t+1} \vert \mathbf{u}_t) \ge \mathcal{L}_t^{\text{bound}}(\mathbf{x}_t,\mathbf{u}_t,\mathbf{x}_{t+1}) \nonumber\\
&\hspace{.2cm}= \mathbb{E}_{q(\mathbf{z}_t | \mathbf{x}_t)q(\mathbf{\hat{z}}_{t+1} | \mathbf{z}_t, \mathbf{u}_t)} \big [\log p(\mathbf{x}_t|\mathbf{z}_t) \nonumber \\  
&\hspace{.2cm}\phantom{{}={}}+\log p(\mathbf{x}_{t+1}|\hat{\mathbf{z}}_{t+1})\big] - \text{KL}\big(q(\mathbf{z}_t | \mathbf{x}_t) \parallel p(\mathbf{z}_t)\big).  
\end{align}
Local linearization of the dynamics is enforced inside the recognition model $q(\mathbf{\hat{z}}_{t+1} | \mathbf{z}_t, \mathbf{u}_t)$, where mapping from a linearization point $\bar{\mathbf{z}}_t$ to the linearization matrices are estimated via neural networks.

Finally, we want $\mathbf{z}_{t+1}$ to be both the embedding of $\mathbf{x}_{t+1}$ and the result of applying $\mathbf{u}_t$ to $\mathbf{z}_t$. E2C attempts to enforce this temporal consistency criterion by encouraging the distributions of $\hat{\mathbf{z}}_{t+1}$ and the next step embedding $\mathbf{z}_{t+1}$ to be similar (in the KL sense). Enforcing the temporal consistency leads to the modified objective
\begin{equation}
\mathcal{L}_t = \mathcal{L}_t^{\text{bound}} - \lambda \text{KL}\big(q(\hat{\mathbf{z}}_{t+1}| \mathbf{z}_t, \mathbf{u}_t) \parallel q(\mathbf{z}_{t+1}|\mathbf{x}_{t+1})\big),
\end{equation}
where $\lambda$ is an additional hyperparameter of the model. We note that neither of the two objectives $\sum_t \mathcal{L}_t^{\text{bound}}$ and $\sum_t \mathcal{L}_t$ is a lower-bound of the data likelihood $p(\mathbf{X})$. The fact that E2C does not optimize a proper lower-bound of the data likelihood has also been observed by~\cite{karl2016deep}.

Compared to E2C, our method is based on introducing a graphical model that clearly separates the generative model from the variational recognition model. This enables us to handle noise in the system and avoid heuristic terms in the objective functions that need extra hyperparameter tuning. Furthermore, we can optimize a lower-bound on the likelihood of the data sequence using a better-designed recognition model more robust w.r.t.~noise. Note that our goal is not to purely obtain the best predictive power as in~\cite{karl2016deep}, but to design a predictive model that yields a suitable embedding representation for locally optimal control. Unlike~\cite{karl2016deep} that does not report control performance, our experiments focus on the performance of the controller under various noise regimes. In the next section, we describe our proposed RCE model and demonstrate how it addresses the aforementioned issues of E2C. We give a deconstruction of the E2C equations to provide its graphical model in the appendix \ref{sec:appendix-E2C}.


\section{Model Description}
\label{sec:RCE}


In this section, we first introduce our graphical model that represents the relation between the observations and latent variables in our model. We then derive a lower-bound on the likelihood of the observation sequence. The objective of training in our model is to maximize this lower-bound. Finally, we describe the details of the method we use for planning in the latent space learned by our model. 


\subsection{Graphical Model}
\label{subsec:GM}

We propose to learn an action-conditional density model of the observations $\mathbf{x}_{1:N}$. Similar to E2C, we assume that the observation sequence is Markovian. Thus, optimizing the likelihood $p(\mathbf{x}_{1:N}\vert \mathbf{u}_{1:N})$ reduces to learning an action-conditional generative model that can be trained via maximum likelihood, i.e.,
\begin{equation}
\max_\theta\;\log p_\theta(\mathbf{x}_{t + 1} | \mathbf{x}_t, \mathbf{u}_t),
\label{eq:loglike}
\end{equation}
where the prediction of the next observation $\mathbf{x}_{t+1}$ depends only on the current $\mathbf{x}_t$ and action $\mathbf{u}_t$. Note that our generative model is parameterized by $\theta$. For notational simplicity, we shall omit $\theta$ in our presentation.

We first discuss how to learn a low-dimensional representation of $\mathbf{x}$ that adheres to globally linear dynamics by incorporating several constraints into the structure of our generative model. First, we introduce the latent variables $\mathbf{z}_t$ and $\mathbf{\hat{z}}_{t+1}$ that serve as information bottlenecks between $\mathbf{x}_t$ and $\mathbf{x}_{t+1}$, such that
\begin{align}
&p(\mathbf{x}_{t + 1}, \mathbf{z}_{t}, \mathbf{\hat{z}}_{t+1} | \mathbf{x}_t, \mathbf{u}_t) \nonumber\\
&\hspace{0.2cm}=p(\mathbf{z}_{t}| \mathbf{x}_t) p(\mathbf{\hat{z}}_{t+1}| \mathbf{z}_{t}, \mathbf{u}_t) p(\mathbf{x}_{t + 1} | \mathbf{\hat{z}}_{t+1}).
\end{align}
Intuitively, it is natural to interpret $\mathbf{z}_t$ and $\mathbf{\hat{z}}_{t+1}$ to be stochastic embeddings of $\mathbf{x}_t$ and $\mathbf{x}_{t+1}$, respectively.

Next, we enforce global linearity of $p(\mathbf{\hat{z}}_{t+1}| \mathbf{z}_{t}, \mathbf{u}_t)$ by restricting it to be a deterministic, linear transition function of the form
\begin{align}
\hat{\mathbf{z}}_{t+1} = \mathbf{A}\mathbf{z}_t + \mathbf{B}\mathbf{u}_t + \mathbf{c},
\end{align}
where $\mathbf{A}$, $\mathbf{B}$, and $\mathbf{c}$ are matrices that respectively define the state dynamics, control dynamics, and the offset. To emphasize the deterministic nature of this transition, we replace all the subsequent mentions of deterministic $p(\cdot | \cdot)$ transitions with $\delta(\cdot | \cdot)$.

In order to learn more expressive transition dynamics, we relax the global linearity constraint to a local one. Unlike global linearity, local linearity requires a linearization point. To account for this, we introduce additional variables $\mathbf{\bar{z}}_t$ and $\mathbf{\bar{u}}_t$ to serve as the linearization point, which results in a new generative model (see the black arrows in Figure~\ref{fig:models}),
\begin{align}
\label{eq:gen-model}
p(&\mathbf{x}_{t + 1}, \mathbf{z}_{t}, \mathbf{\bar{z}}_{t}, \mathbf{\hat{z}}_{t+1} | \mathbf{x}_t, \mathbf{u}_t, \mathbf{\bar{u}}_t) = \nonumber \\
&p(\mathbf{z}_{t}| \mathbf{x}_t)  p(\mathbf{\bar{z}}_{t}| \mathbf{x}_t) 
 \delta(\mathbf{\hat{z}}_{t+1}| \mathbf{z}_{t}, \mathbf{\bar{z}}_{t}, \mathbf{u}_t, \mathbf{\bar{u}}_t)  p(\mathbf{x}_{t + 1} | \mathbf{\hat{z}}_{t+1}),
\end{align}
whose corresponding deterministic transition function for $\delta(\mathbf{\hat{z}}_{t+1}| \mathbf{z}_{t}, \mathbf{\bar{z}}_{t}, \mathbf{u}_t, \mathbf{\bar{u}}_t)$ is
\begin{equation}
\label{eq: latent_lin_dyn}
\hat{\mathbf{z}}_{t+1} = \mathbf{A}_t(\bar{\mathbf{z}}_t,\mathbf{\bar{u}}_t)\mathbf{z}_t + \mathbf{B}_t(\bar{\mathbf{z}}_t,\mathbf{\bar{u}}_t)\mathbf{u}_t + \mathbf{c}_t(\bar{\mathbf{z}}_t,\mathbf{\bar{u}}_t).
\end{equation}
Here, $\mathbf{A}$, $\mathbf{B}$, and $\mathbf{c}$ are functions of $(\bar{\mathbf{z}}_t,\mathbf{\bar{u}}_t)$, and can be parameterized by neural networks. Since the linearization variable $\mathbf{\bar{z}}_t$ is not known in advance, we treat $\mathbf{\bar{z}}_t$ as a random variable with distribution $p(\mathbf{\bar{z}}_{t}| \mathbf{x}_t)$. A natural consideration for $p(\mathbf{\bar{z}}_{t}| \mathbf{x}_t)$ is to set it to be identical to $p(\mathbf{z}_{t}| \mathbf{x}_t)$ \emph{a priori}. This has the effect of making the iLQR controller robust to stochastic sampling of $z_{t}$ during planning. The linearization variable $\mathbf{\bar{u}}_t$ can be obtained from a local perturbation of action $\mathbf{u}_t$.

\begin{figure}[!b]
    \centering
    \includegraphics[trim = 0mm 7mm 0mm 10mm,height=4.7cm]{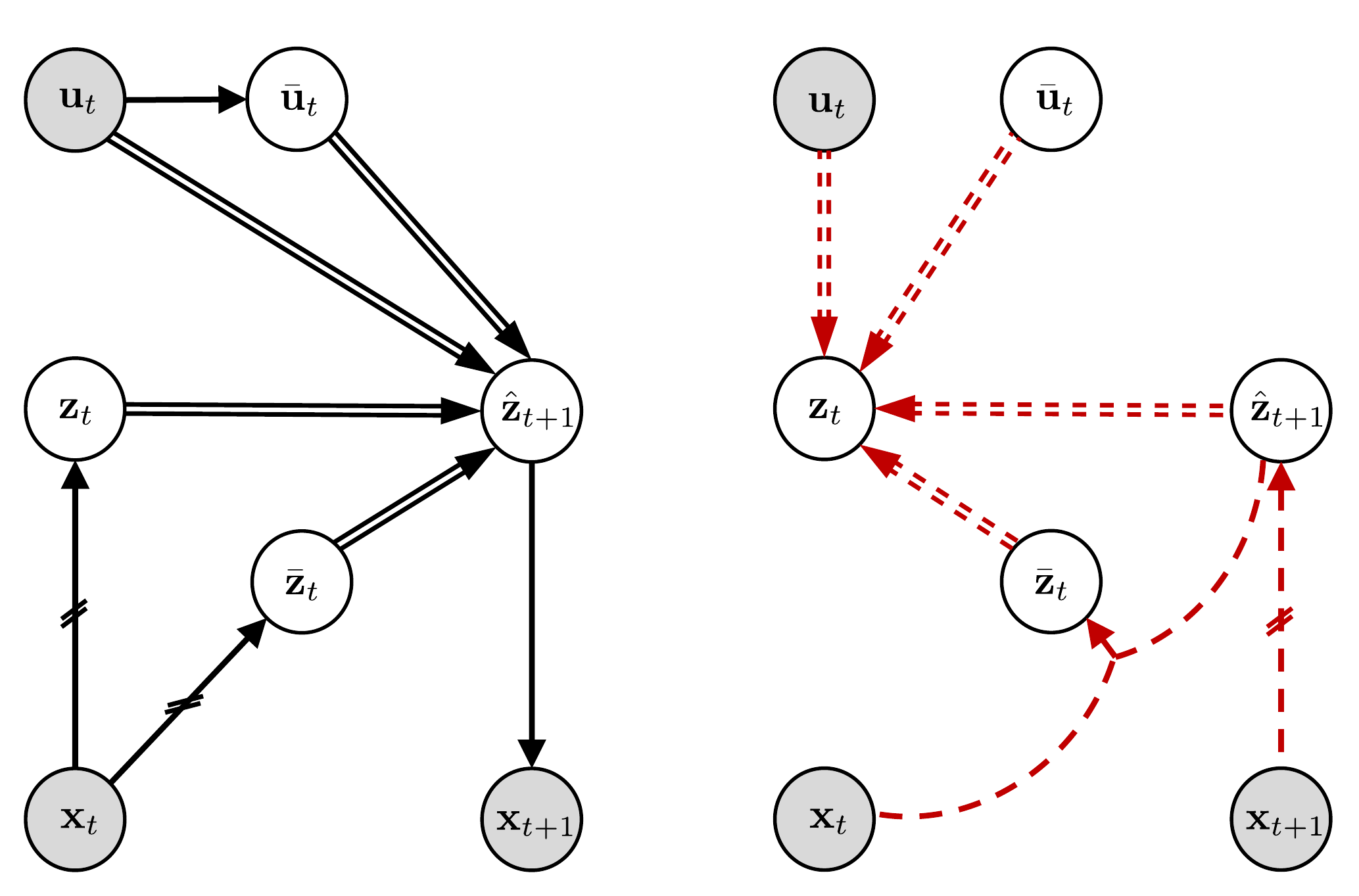} 
    \caption{RCE graphical model. Left: generative links and, right: recognition model. Parallel lines mean deterministic links, while single lines mean stochastic links (a link that involves in sampling). $\mathbf{z}_t $ and $\bar{\mathbf{z}}_t$ are two samples from $p(\mathbf{z}|\mathbf{x})$. We use a single network (the encoder network) to model the conditional probability of the links with the hatch marks. } \label{fig:models}%
\end{figure}


\subsection{Deep Variational Learning}
\label{subsec:DVL}

Training the generative model in Eq.~\ref{eq:gen-model} using maximum likelihood is intractable, since it requires the marginalization of the latent variables. Therefore, we propose to use deep variational inference~\cite{vae2014,rezende2014stochastic} and maximize the variational lower-bound of the log-likelihood, instead. The variational lower-bound requires us to define a variational approximation to the true posterior 

\vspace{-0.15in}
\begin{small}
\begin{align*}
q(\mathbf{z}, \mathbf{\bar{z}}_t, \mathbf{\hat{z}}_{t+1} | \mathbf{x}_t, \mathbf{x}_{t+1}, \mathbf{u}_t, \mathbf{\bar{u}}_t) \approx
p(\mathbf{z}, \mathbf{\bar{z}}_t, \mathbf{\hat{z}}_{t+1} | \mathbf{x}_t, \mathbf{x}_{t+1}, \mathbf{u}_t, \mathbf{\bar{u}}_t).
\end{align*}
\end{small}
\vspace{-0.2in}

In adherence to the interpretation of $\mathbf{z}_t$ and $\mathbf{\hat{z}}_{t+1}$ as stochastic embeddings of $\mathbf{x}_t$ and $\mathbf{x}_{t+1}$, it is important to enforce consistency between $p(\mathbf{\hat{z}}_{t+1} | \mathbf{x}_t, \mathbf{x}_{t+1}, \mathbf{u}_t, \mathbf{\bar{u}}_t)$ and the next step probability of the embedding given the observation $p(\mathbf{z}_{t+1} | \mathbf{x}_{t+1})$. Since we do not have access to $p(\mathbf{\hat{z}}_{t+1} |\mathbf{x}_t, \mathbf{x}_{t+1}, \mathbf{u}_t, \mathbf{\bar{u}}_t)$, we instead encourage this consistency through posterior regularization by explicitly setting
\begin{align}
\label{eq:pequalq}
q_\phi(\mathbf{\hat{z}}_{t+1} | \mathbf{x}_{t+1}) = p(\mathbf{z}_{t+1} | \mathbf{x}_{t+1}).
\end{align}
Next, we propose a novel factorization of the full variational posterior as
\begin{align}
\label{eq:q_fac}
&q(\mathbf{z}, \mathbf{\bar{z}}_t, \mathbf{\hat{z}}_{t+1} | \mathbf{x}_t, \mathbf{x}_{t+1}, \mathbf{u}_t, \mathbf{\bar{u}}_t) \\
&\hspace{0.2cm}= 
q_\phi(\mathbf{\hat{z}}_{t+1} | \mathbf{x}_{t+1}) 
q_\varphi(\mathbf{\bar{z}}_{t} | \mathbf{x}_{t}, \mathbf{\hat{z}}_{t+1}) 
\delta(\mathbf{z}_{t} | \mathbf{\hat{z}}_{t+1}, \mathbf{\bar{z}}_t, \mathbf{u}_t, \mathbf{\bar{u}}_t), \nonumber
\end{align}
where $q_\varphi(\mathbf{\bar{z}}_{t} | \mathbf{x}_{t}, \mathbf{\hat{z}}_{t+1})$ is the \emph{backward encoder} and $\delta(\mathbf{z}_{t} | \mathbf{\hat{z}}_{t+1}, \mathbf{\bar{z}}_t, \mathbf{u}_t, \mathbf{\bar{u}}_t)$ is the deterministic \emph{reverse transition}. Our choice of factorization results in a recognition model that contrasts sharply with that in E2C. First, our recognition model properly conditions the inference of $\hat{\mathbf{z}}_{t+1}$ on the observation $\mathbf{x}_{t+1}$. Second, our recognition model explicitly accounts for the deterministic transition in the generative model; inference of the deterministic transition can be performed in closed-form using a deterministic reverse transition that recovers $\mathbf{z}_{t}$ as a function of $\mathbf{\bar{z}}_t$, $\mathbf{u}_t$, $\bar{\mathbf{u}}_t$, and $\mathbf{\hat{z}}_{t+1}$. To be consistent with Eq.~\ref{eq: latent_lin_dyn}, we require that
\begin{align*}
\mathbf{z}_t = \mathbf{A}_t^{-1}(\bar{\mathbf{z}}_t,\bar{\mathbf{u}}_t)\big(\hat{\mathbf{z}}_{t+1} - \mathbf{B}_t(\bar{\mathbf{z}}_t,\bar{\mathbf{u}}_t)\mathbf{u}_t - \mathbf{c}_t(\bar{\mathbf{z}}_t,\bar{\mathbf{u}}_t)\big).
\end{align*}
During the training of the generative model, we only need to access the inverse of $\mathbf{A}_t$. As such, we propose to directly train a network that outputs its inverse $\mathbf{M}_t(\bar{\mathbf{z}}_t,\bar{\mathbf{u}}_t) = \mathbf{A}_t^{-1}(\bar{\mathbf{z}}_t,\bar{\mathbf{u}}_t)$. To make sure that $\mathbf{M}_t$ is an invertible matrix and to enable efficient estimation, we restrict $\mathbf{M}_t$ to be a rank-one perturbation of the identity matrix, i.e.,
\begin{equation}
\mathbf{M}_t = \mathbf{I}_{n_z} + \mathbf{w}_t(\bar{\mathbf{z}}_t,\bar{\mathbf{u}}_t)\cdot\mathbf{r}_t(\bar{\mathbf{z}}_t,\bar{\mathbf{u}}_t)^\top,
\label{eq:M}
\end{equation}
where $\mathbf{I}_{n_z}$ is the identity matrix of size $n_z$, and $\mathbf{w}_t$ and $\mathbf{r}_t$ are two column vectors in $\mathbb{R}^{n_z}$. We constraint these vectors to be non-negative using a non-negative activation at their corresponding output layers.


We now formally define the RCE loss and its lower-bound property.
\begin{lemma}
\label{lm:lem}
Let $\mathcal{L}_t^{\text{RCE}}$ be defined as
\begin{align}
\label{eq:RCE_obj_func_1}
&\mathcal{L}_t^{\text{RCE}}  = \mathbb{E}_{q_{\phi}(\hat{\mathbf{z}}_{t+1}|\mathbf{x}_{t+1})} \big [\log p(\mathbf{x}_{t+1}|\hat{\mathbf{z}}_{t+1}) \big ] \\
&- \mathbb{E}_{q_{\phi}(\hat{\mathbf{z}}_{t+1}|\mathbf{x}_{t+1})} \big [ \text{KL} \big ( q_{\varphi}(\bar{\mathbf{z}}_t| \hat{\mathbf{z}}_{t+1},\mathbf{x}_t) \parallel p(\bar{\mathbf{z}}_t| \mathbf{x}_t)  \big ) \big ] \nonumber \\
&+ \text{H} \big ( q_{\phi}(\hat{\mathbf{z}}_{t+1}| \mathbf{x}_{t+1}) \big )  + \mathbb{E}_{\substack{ q_{\phi}(\hat{\mathbf{z}}_{t+1}|\mathbf{x}_{t+1})\\q_{\varphi}(\bar{\mathbf{z}}_t|\mathbf{x}_t,\hat{\mathbf{z}}_{t+1}) } }\big [ \log p(\mathbf{z}_t| \mathbf{x}_t) \big ]. \nonumber
\end{align}
Subject to the constraints we explicitly impose on $q$, $\mathcal{L}_t^{\text{RCE}}$ is a lower-bound on the conditional log-likelihood $\log p(\mathbf{x}_{t+1}| \mathbf{x}_t, \mathbf{u}_t, \bar{\mathbf{u}}_t)$, which in trun defines a lower-bound on the conditional likelihood of our interest, i.e.,~$p(\mathbf{x}_{t+1}| \mathbf{x}_t, \mathbf{u}_t)=\int p(\mathbf{x}_{t+1}| \mathbf{x}_t, \mathbf{u}_t, \bar{\mathbf{u}}_t)\;p(\bar{\mathbf{u}}_t | \mathbf{u}_t)\;d\bar{\mathbf{u}}_t$.
\end{lemma}
\begin{proof}
See Appendix~\ref{sec:obj-func}.
\end{proof}


\begin{figure*}[!t]
\centering
\includegraphics[trim = 0mm 0mm 0mm 10mm,width=16cm]{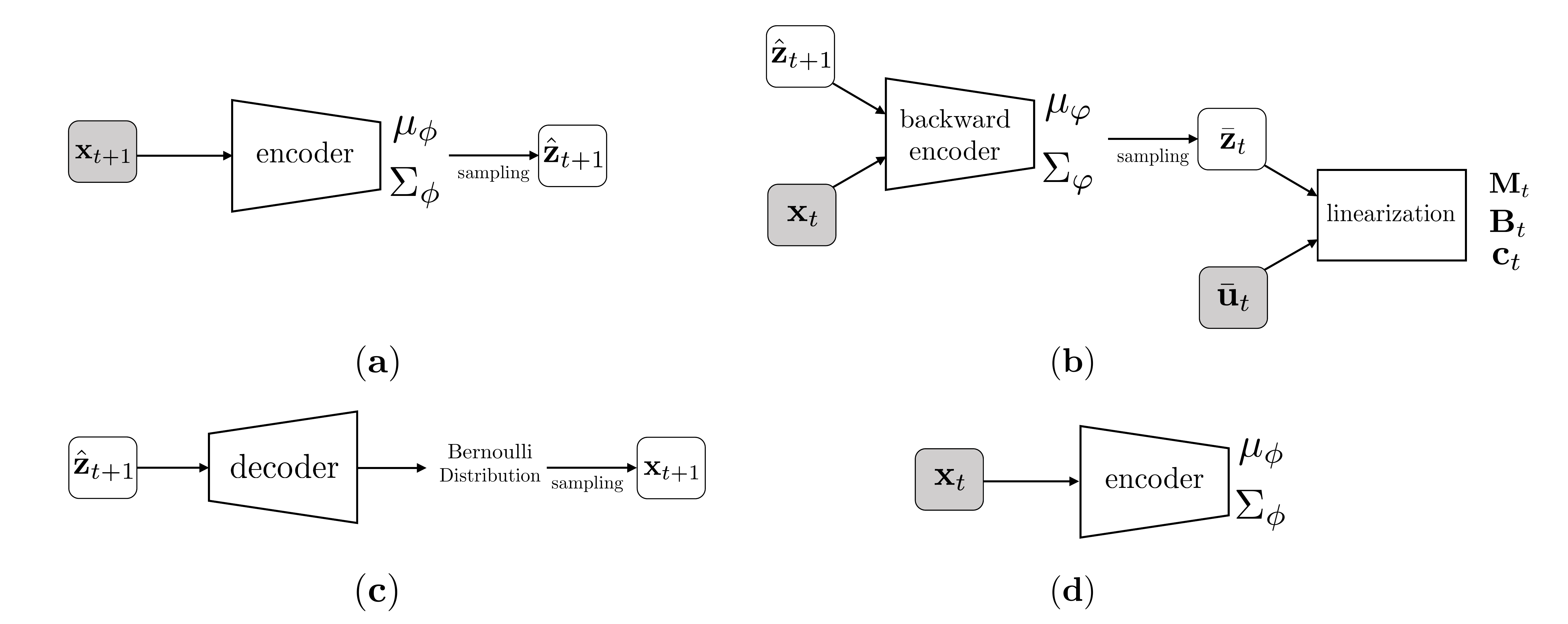} 
\caption{Schematic of the networks that are used for modeling the probabilities in our model. The gray boxes contain input (observable) variables. \textbf{(a) }Encoder network that models $q_{\phi}(\hat{\mathbf{z}}_{t+1}|\mathbf{x}_{t+1}) = \mathcal{N}\big(\mu_{\phi}(\mathbf{x}_{t+1}),\mathbf{\Sigma}_{\phi}(\mathbf{x}_{t+1})\big)$.  \textbf{(b)} Transition network that contains two parts. One part, denoted by ``backward encoder", models  $q_{\varphi}(\bar{\mathbf{z}}_t|\mathbf{x}_t,\hat{\mathbf{z}}_{t+1}) = \mathcal{N}\big(\mu_{\varphi}(\mathbf{x}_t,\hat{\mathbf{z}}_{t+1}),\mathbf{\Sigma}_{\varphi}(\mathbf{x}_t,\hat{\mathbf{z}}_{t+1})\big)$, and the other part, denoted by ``linearization", is used to obtain $\mathbf{M}_t$, $\mathbf{B}_t$, and $\mathbf{c}_t$, which are the parameters of the locally linear model in the latent space. \textbf{(c)} Decoder network that models $p(\mathbf{x}_{t+1}|\hat{\mathbf{z}}_{t+1})$. In our experiments we assume that this distribution is Bernoulli. Therefore, we use sigmoid nonlinearity at the last layer of the decoder. $\bar{\mathbf{x}}_{t+1}$ is the reconstructed version of $\mathbf{x}_{t+1}$. \textbf{(d)} The network that models $p(\mathbf{z}_t|\mathbf{x}_t)$. According to Eq. \ref{eq:pequalq}, since  $p(\mathbf{z}_t|\mathbf{x}_t) = q_{\phi}(\mathbf{z}_t|\mathbf{x}_t)$ and therefore  we tie the parameters of this network with the encoder network, $p(\mathbf{z}_t|\mathbf{x}_t) = \mathcal{N}(\mu_{\phi}(\mathbf{x}_t),\mathbf{\Sigma}_{\phi}(\mathbf{x}_t))$. Note that $p(\mathbf{z}_t|\mathbf{x}_t) $ is the same as $p(\bar{\mathbf{z}}_t|\mathbf{x}_t)$. Thus, the KL term in~\eqref{eq:RCE_obj_func_1} can be written as $\text{KL}\big(\mathcal{N}(\mu_{\varphi},\mathbf{\Sigma}_{\varphi}) \parallel \mathcal{N}(\mu_{\phi}(\mathbf{x}_t),\mathbf{\Sigma}_{\phi}(\mathbf{x}_t))\big)$.
}
\label{fig: network_structure}
\end{figure*}

Figure~\ref{fig:models} contains a graphical representation of our model. It is important to note that unlike the E2C encoder (Eq.~8 in~\cite{watter2015embed}), our recognition model takes the future state $\mathbf{x}_{t+1}$ as input. In the case of noisy dynamics, the future state heavily influences the posterior. Thus, E2C's failure to incorporate the future state into the variational approximation of the posterior could be detrimental to the performance of the system in the noisy regime. We clearly demonstrate this phenomenon in our experiments.

\subsection{Network Structure}
\label{subsec:Network}

For the four problems used in our experiments in Section~\ref{sec:experiments}, we use feedforward networks for encoding, decoding, and transition. Depending on the input image size, the encoder and decoder can have fully-connected layers or convolutional layers. The transition networks always have fully-connected layers. According to Eq.~\ref{eq:RCE_obj_func_1}, we need to model four different conditional probabilities: $p(\mathbf{x}_{t+1}|\hat{\mathbf{z}}_{t+1})$, $q_{\phi}(\hat{\mathbf{z}}_{t+1}| \mathbf{x}_{t+1}) $, $q_{\varphi}(\bar{\mathbf{z}}_t| \hat{\mathbf{z}}_{t+1},\mathbf{x}_t)$, and $p(\mathbf{z}_t| \mathbf{x}_t)$. Figure~\ref{fig: network_structure} shows the high-level depiction of the networks and the connection between different variables used in these probabilities.


\subsection{Planning}
\label{subsec:planning}

We use the iLQR algorithm to plan in the latent space $\mathcal{Z}$. A good latent space representation should allow us not only to reconstruct and predict the images accurately, but also to plan well in this space using $f_{\mathcal{Z}}$. 


The inputs to the planning algorithm are the two high-dimensional observations $\mathbf{x}^i$ and $\mathbf{x}^f$, corresponding to the initial and final (goal) states $\mathbf{s}^i$ and $\mathbf{s}^f$. We encode these two high-dimensional observations to the latent space observations $\mathbf{z}^i$ and $\mathbf{z}^f$. We sample a random set of $H$ actions $\mathbf{u}_{1:H}$ and apply them to the model we have learned in the latent space $\mathcal{Z}$, starting from the initial observation $\mathbf{z}^i$. This generates a reference trajectory $\{\bar{\mathbf{z}}_1=\mathbf{z}^i,\bar{\mathbf{u}}_1=\mathbf{u}_1,\bar{\mathbf{z}}_2,\bar{\mathbf{u}}_2=\mathbf{u}_2,\ldots,\bar{\mathbf{z}}_H,\bar{\mathbf{u}}_H=\mathbf{u}_H,\bar{\mathbf{z}}_{H+1}\}$ of size $H$. We pass this reference trajectory to iLQR and it returns the set of actions $\mathbf{u}^*_{1:H}$ that has been iteratively optimized to minimize a quadratic cost similar to~\eqref{eq:SLOC} in the latent space $\mathcal{Z}$. We apply $\mathbf{u}^*_1$ to the dynamical system, observe the next state's observation $\mathbf{x}_2$, and encode it to the latent space observation $\mathbf{z}_2$. We then generate another reference trajectory by starting from $\mathbf{z}_2$ and applying the sequence of $H$ actions $\{\mathbf{u}^*_2,\ldots,\mathbf{u}^*_{H},\mathbf{u}_{H+1}\}$, where $\mathbf{u}_{H+1}$ is a random action. We then run iLQR with this trajectory and apply the first action in the set of $H$ actions it returns to the system. We continue this process for $T$ (the planning horizon) steps.

\section{Experiments}
\label{sec:experiments}

In this section, we compare the performance of our proposed RCE model with that of E2C in terms of both prediction and planning in the four domains of~\cite{watter2015embed}. To generate our training and test sets, each consists of triples $(\mathbf{x}_t,\mathbf{u}_t,\mathbf{x}_{t+1})$, we first sample a state $\mathbf{s}_t$ and generate its corresponding observation $\mathbf{x}_t$. We then take an action $\mathbf{u}_t$ and add a Gaussian noise with covariance $\mathbf{\Sigma}_{\mathbf{n}^{\mathcal{S}}}$ according to Eq.~\ref{eq:true_dynamic} to obtain the next state $\mathbf{s}_{t+1}$, which is used to generate the next observation $\mathbf{x}_{t+1}$. We consider both deterministic ($\mathbf{\Sigma}_{\mathbf{n}^{\mathcal{S}}} = \textbf{0}$) and stochastic scenarios. In the stochastic case, we add noise to the system with different values of $\mathbf{\Sigma}_{\mathbf{n}^{\mathcal{S}}}$ and evaluate the models performance under noise. 

In each of the four domains used in our experiments, we compare the performance of RCE and that of E2C in terms of four different factors (see Tables~\ref{tbl:planar_sys}--~\ref{tbl:arm}). {\bf 1)} {\em Reconstruction Loss} is the loss in reconstructing $\mathbf{x}_t$ using the encoder and decoder. {\bf 2)} {\em Prediction Loss} is the loss in predicting $\mathbf{x}_{t+1}$, given $\mathbf{x}_t$ and $\mathbf{u}_t$, using the encoder, decoder, and transition network. Both reconstruction and prediction loss are computed based on binary cross entropy. {\bf 3)} {\em Planning Loss} is computed based on the following quadratic loss:
\begin{equation}
J = \sum \limits_{t=1}^T (\mathbf{s}_t - \mathbf{s}^f)^{\top} \mathbf{Q} (\mathbf{s}_t - \mathbf{s}^f) + \mathbf{u}_t^{\top} \mathbf{R}\mathbf{u}_t.
\label{eq:traj_loss}
\end{equation}

\begin{table*}[!b]
\caption{RCE and E2C Comparison -- Planar System}
\vspace{-.25cm}
\small
\begin{center}
\begin{tabular}{ c c | c c c c}
\textbf{$\mathbf{\Sigma}_{\mathbf{n}^{\mathcal{S}}}$} & \textbf{Algorithm} & \textbf{Reconstruction Loss} & \textbf{Prediction Loss} & \textbf{Planning Loss} & \textbf{Success Rate} \\
\hline
	  \multirow{ 2}{*}{0 }	 		&\textbf{RCE }	&  $3.6 \pm 1.7$ & $6.2 \pm 2.8$  & $21.4  \pm 2.9$ & $100 \% $\\
&\textbf{E2C} & $7.4 \pm 1.7$ & $9.3 \pm 2.8$ & $26.3 \pm 4.9$ & $100 \%$ \\
\cline{1-6}
  \multirow{ 2}{*}{1}	 		&\textbf{RCE }	 & $8.3 \pm 5.5$ & $10.1 \pm 6.2$  & $25.4 \pm 3.6$ & $100 \% $\\
																			 	&\textbf{E2C}    & $19.2 \pm 5.1$ & $28.3 \pm 10.2$   & $34.1 \pm 9.5$  & $95 \%$ \\
\cline{1-6}
 \multirow{ 2}{*}{2}	 &\textbf{RCE }	&  $12.3 \pm 4.9$ & $17.3 \pm 6.2$  & $36.4  \pm 10.3$ & $95 \% $\\
						   	 &\textbf{E2C}  & $37.1 \pm 10.5$ & $ 45.8 \pm 13.1$   & $63.7 \pm 16.3$  & $75 \%$  \\
\cline{1-6}
 \multirow{ 2}{*}{5 }	 &\textbf{RCE }	&  $25.2 \pm 6.1$ & $27.3 \pm 8.2$  & $50.3 \pm 14.5$ & $85 \% $\\
							 &\textbf{E2C}      & $60.3 \pm 18.2$ & $78.3 \pm 15.0$   & $112.4 \pm 30.2$  & $45 \%$  \\
							 
\end{tabular}
\end{center}
\label{tbl:planar_sys}
\vspace{-.3cm}
\end{table*}

\begin{figure*}[!b]
\centering
\includegraphics[trim = 0mm 10mm 0mm 14mm,width=15.5cm]{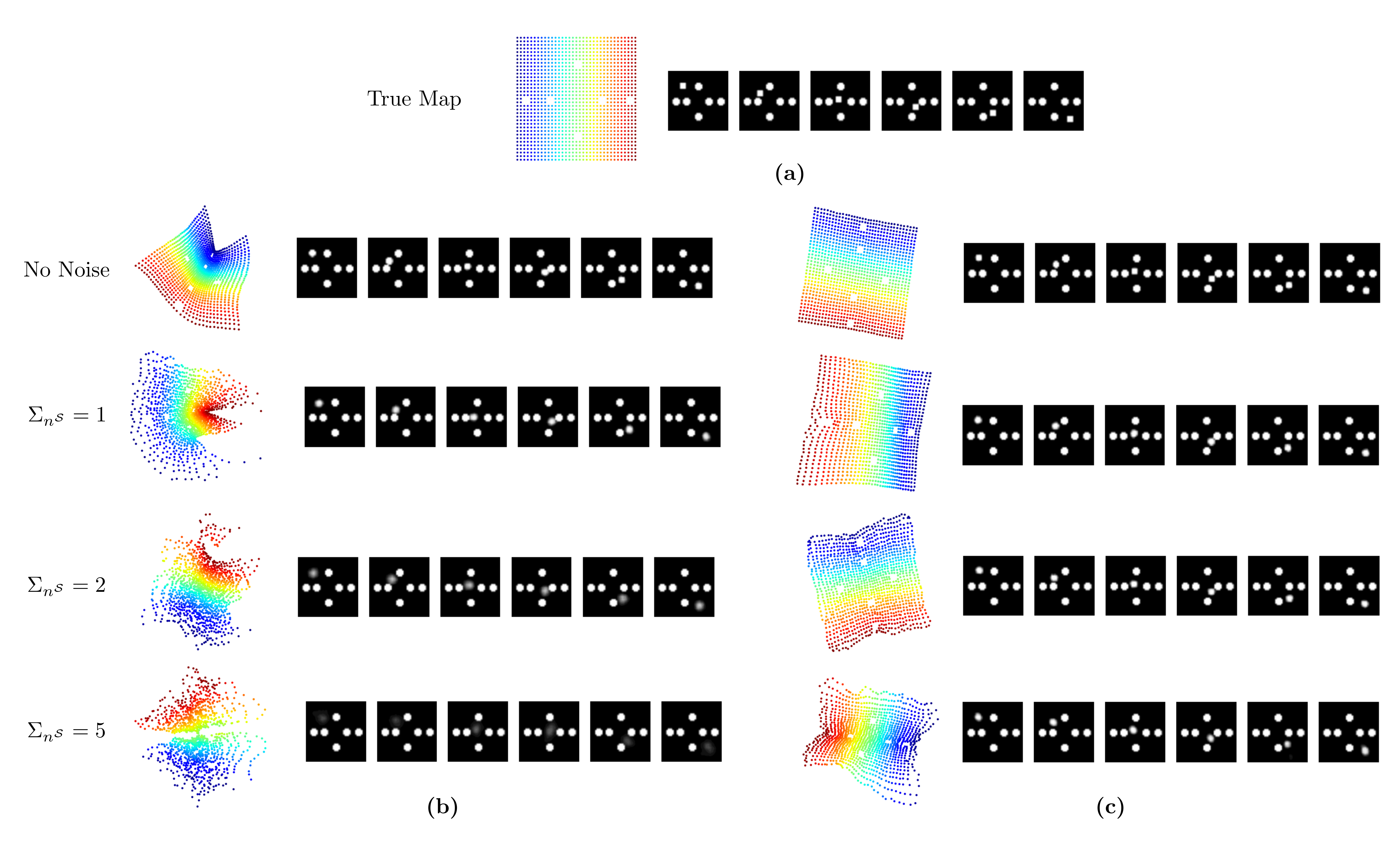} 
\vspace{-.2cm}
\caption{{\em (a) Left:} The true state space of the planar system. Each point on the map corresponds to one image in the dataset. {\em (a) Right:} A random trajectory. Each image is $40 \times 40$ black and white. The circles show the obstacles and the square is the agent in the domain. {\em (b)} Reconstructed map and predicted trajectory in the latent space of the E2C model for different noise levels. {\em (c)} Reconstructed map and predicted trajectory in the latent space of the RCE model for different noise levels.}
\label{fig: maptraj}
\end{figure*}

We apply the sequence of actions returned by iLQR to the dynamical system and report the value of the loss in Eq.~\ref{eq:traj_loss}. {\bf 4)} {\em Success Rate} shows the number of times the agents reaches the goal within the planning horizon $T$, and remains near the goal in case it reaches it in less than $T$ steps. For each of the domains, all the results are averaged over $20$ runs. The details of our implementations, including the network's structure, the size of the latent space, and the planning horizon are specified in Appendix~\ref{sec:implement}.

\vspace{-.2cm}
\subsection{Planar System} 
Consider an agent in a surrounded area, whose goal is to navigate from a corner to the opposite one, while avoiding the six obstacles in this area. The system is observed through a set of $40 \times 40$ pixel images taken from the top, which specify the agent's location in the area. Actions are two-dimensional and specify the direction of the agent's movement. 


Table~\ref{tbl:planar_sys} shows that RCE outperforms E2C in both prediction/reconstruction and planning in this domain. The Gaussian noise we add to the system has a diagonal covariance matrix with equal variance in all dimensions. The values mentioned in the table for $\mathbf{\Sigma}_{\mathbf{n}^{\mathcal{S}}}$ are the standard deviation in each dimension.

Figure~\ref{fig: maptraj} shows the latent space representation of data points in the planar system dataset for both RCE and E2C models. RCE has clearly a more robust representation against the noise and is able to predict the defined trajectory with a much higher quality.


\begin{table*}[!t]
\caption{RCE and E2C Comparison -- Inverted Pendulum (Acrobat)}
\vspace{-.35cm}
\small
\begin{center}
\begin{tabular}{ c c | c c c c}
\textbf{$\mathbf{\Sigma}_{\mathbf{n}^{\mathcal{S}}}$} & \textbf{Algorithm} & \textbf{Reconstruction Loss} & \textbf{Prediction Loss} & \textbf{Planning Loss} & \textbf{Success Rate} \\
\hline
 \multirow{ 2}{*}{0}	 & \textbf{RCE }	&  $43.1 \pm 13.2$ & $48.1 \pm 17.2$  & $14.2  \pm 4.6$ & $95 \% $\\
&\textbf{E2C} & $73.2 \pm 20.1$ & $79.6 \pm 32.6$ & $16.1 \pm 2.9$ & $90 \%$  \\
\cline{1-6}
  \multirow{ 2}{*}{1 }	 &\textbf{RCE }		&  $61.1 \pm 16.2$ & $70.2 \pm 36.1$  & $17.3  \pm 7.1$ & $85 \% $\\
  					 &\textbf{E2C}    	& $97.1 \pm 34.1$ & $109.7 \pm 58.2$   & $29.9 \pm 9.2$  & $60 \%$ \\
\cline{1-6}
 \multirow{ 2}{*}{2 }	 &\textbf{RCE }	&  $92.11 \pm 35.4$ & $106.4 \pm 53.2$  & $27.5  \pm 6.6$ & $70 \% $\\
						   		 &\textbf{E2C}  & $140.2 \pm 47.1$ & $179.5 \pm 61.1$   & $40.7 \pm 11.8$  & $40 \%$  \\
\end{tabular}
\end{center}
\label{tbl:acrobat}
\vspace{-.6cm}
\end{table*}

\begin{table*}[!b]
\vspace{-.2cm}
\caption{RCE and E2C Comparison -- Cart-pole Balancing}
\vspace{-.35cm}
\small
\begin{center}
\begin{tabular}{ c c | c c c c}
\textbf{$\mathbf{\Sigma}_{\mathbf{n}^{\mathcal{S}}}$} & \textbf{Algorithm} & \textbf{Reconstruction Loss} & \textbf{Prediction Loss} & \textbf{Planning Loss} & \textbf{Success Rate} \\
\hline
 \multirow{ 2}{*}{0}	 & \textbf{RCE }	&  $33.2 \pm 15.6$ & $42.1 \pm 26.9$  & $21.2  \pm 6.3$ & $90 \% $\\
																&\textbf{E2C} 		& $44.9 \pm 17.0$  & $57.3 \pm 22.9$  & $25.3 \pm 4.8$ 	& $85 \%$  \\
\cline{1-6}
  \multirow{ 2}{*}{1 }	 &\textbf{RCE }		&  $52.1 \pm 20.3$ & $63.3 \pm 27.2$  & $28.4  \pm 5.5$ & $80 \% $\\
						 &\textbf{E2C}    	&  $70.2 \pm 23.7$  & $90.5 \pm 42.4$ & $39.8 \pm 5.2$  & $70 \%$ \\
\cline{1-6}
 \multirow{ 2}{*}{2 }	 &\textbf{RCE }	&  $77.6 \pm 30.2$ & $88.4 \pm 38.3$  & $42.2  \pm 8.3$ & $70 \% $\\
				 		 &\textbf{E2C}  & $112.6 \pm 39.2$  & $133.0 \pm 56.5$  & $67.2 \pm 9.3$ & $40 \%$  \\
\end{tabular}
\end{center}
\label{tbl:cartpole}
\vspace{-.2cm}
\end{table*}

\begin{table*}[!b]
\vspace{-.2cm}
\caption{RCE and E2C Comparison -- Robot Arm}
\vspace{-.35cm}
\small
\begin{center}
\begin{tabular}{ c c | c c c c}
\textbf{$\mathbf{\Sigma}_{\mathbf{n}^{\mathcal{S}}}$} & \textbf{Algorithm} & \textbf{Reconstruction Loss} & \textbf{Prediction Loss} & \textbf{Planning Loss} & \textbf{Success Rate} \\
\hline
 \multirow{ 2}{*}{0}	 & \textbf{RCE }	&  $60.5 \pm 27.1$ & $69.9 \pm 32.2$  & $81.3  \pm 35.5$ & $90 \% $\\
																&\textbf{E2C} 		& $71.3 \pm 19.5$  & $83.4 \pm 28.6$  & $90.23 \pm 47.38$ 	& $90\%$  \\
\cline{1-6}
  \multirow{ 2}{*}{1 }	 &\textbf{RCE }		&  $96.5 \pm 34.4$ & $112.6 \pm 42.2$  & $106.2  \pm 50.8$ & $80 \% $\\
						 &\textbf{E2C}    	&  $138.1\pm 42.5$  & $172.2 \pm 58.3$ & $155.2 \pm 70.1$  & $65 \%$ \\

\end{tabular}
\end{center}
\label{tbl:arm}
\end{table*}

\vspace{-.2cm}
\subsection{Inverted Pendulum (Acrobat)} 
This is the classic problem of controlling an inverted pendulum~\cite{Wang96AF} from $48 \times 48$ pixel images. The goal in this task is to swing up and balance an underactuated pendulum from a resting position (pendulum hanging down). The true state space of the system $\mathcal{S}$ has two dimensions: angle and angular velocity. To keep the Markovian property in the observation space, we need to have two images in each observation $\mathbf{x}_t$, since each image shows only position of the pendulum and does not have any information about its velocity. 

Table~\ref{tbl:acrobat} contains our results of comparing RCE and E2C models in this task. Learning the dynamics in this problem is harder than reconstructing the images. Therefore, at the beginning of the training we set the weights of the two middle terms in Eq.~\ref{eq:RCE_obj_func_1} to 10, and eventually decrease them to 1. The results show that RCE outperforms than E2C, and the difference is significant under noisy conditions.

\vspace{-.2cm}
\subsection{Cart-pole Balancing} 
This is the visual version of the classic task of controlling a cart-pole system~\cite{Sutton98RL}. The goal in this task is to balance a pole on a moving cart, while the cart avoids hitting the left and right boundaries. The control (action) is 1-dimensional and is the force applied to the cart. The original state of the system $\mathbf{s}_t$ is 4-dimensional. The observation $\mathbf{x}_t$ is a history of two $80 \times 80$ pixel images (to maintain the Markovian property). Due to the relatively large size of the images, we use convolutional layers in encoder and decoder. To make a fair comparison with E2C, we also set the dimension of the latent space $\mathcal{Z}$ to 8.

Table~\ref{tbl:cartpole} contains our results of comparing RCE and E2C models in this task. We again observe a similar trend: RCE outperforms E2C in both noiseless and noisy settings and is significantly more robust. 

\vspace{-.2cm}
\subsection{Three-link Robot Arm} 
\vspace{-.1cm}
The goal in this task is to move a three-link planar robot arm from an initial position to a final position (both chosen randomly). The real state of the system $\mathcal{S}$ is $6$-dimensional and the actions are $3$-dimensional, representing the force applied to each joint of the arm. We use two $128 \times 128$ pixel images of the arm as observation $\mathbf{x}$. To be consistent with the E2C model, we choose the latent space $\mathcal{Z}$ to be $8$-dimensional.

Table~\ref{tbl:arm} contains our results of comparing RCE and E2C models in this task. Similar to the other domains, our results show that the RCE model is more robust to noise than E2C. 

\vspace{-.1cm}
\section{Conclusions}
\label{sec:conclu}
\vspace{-.1cm}
In summary, we proposed a new method to embed the high-dimensional observations of a MDP in such a way that both the embeddings and locally optimal controllers are robust w.r.t.~the noise in the system's dynamics. Our RCE model enjoys a clean separation between the generative graphical model and its recognition model. The RCE's generative model explicitly treats the unknown linearization points as random variables, while the recognition model is factorized in reverse direction to take into account the future observation as well as exploiting determinism in the transition dynamics. Our experimental results demonstrate that the RCE's predictive and planning performance are better and significantly more robust than that of E2C in all the four benchmarks where E2C performance has been measured~\cite{watter2015embed}.


\newpage

\bibliographystyle{abbrv}
\bibliography{Robust_Control_of_High_Dimensional_MDPs}

\newpage
\onecolumn
\appendix
\section{Objective Function}
\label{sec:obj-func}

\begin{proof}[Proof of Lemma~\ref{lm:lem}]

We define $q^*= q(\mathbf{z}_t,\bar{\mathbf{z}}_t,\hat{\mathbf{z}}_{t+1}|\mathbf{x}_t,\mathbf{x}_{t+1}, \mathbf{u}_t, \bar{\mathbf{u}}_t)$. 
Our goal is to define a variational lower-bound on the conditional log-likelihood $\log p(\mathbf{x}_{t+1}| \mathbf{x}_t, \mathbf{u}_t)$. The likelihood $p(\mathbf{x}_{t+1}| \mathbf{x}_t, \mathbf{u}_t)$ may be written as
\begin{align*}
p(\mathbf{x}_{t+1}| \mathbf{x}_t, \mathbf{u}_t) &= \int p(\mathbf{x}_{t+1}, \bar{\mathbf{u}}_t | \mathbf{x}_t, \mathbf{u}_t)\;d\bar{\mathbf{u}}_t = \int \frac{p(\mathbf{x}_{t+1}, \bar{\mathbf{u}}_t, \mathbf{x}_t, \mathbf{u}_t)}{p(\mathbf{x}_t, \mathbf{u}_t)}\;d\bar{\mathbf{u}}_t \\ 
&= \int \frac{p(\mathbf{x}_{t+1} | \bar{\mathbf{u}}_t, \mathbf{x}_t, \mathbf{u}_t)\;p(\bar{\mathbf{u}}_t | \mathbf{x}_t, \mathbf{u}_t)\;p(\mathbf{x}_t, \mathbf{u}_t)}{p(\mathbf{x}_t, \mathbf{u}_t)}\;d\bar{\mathbf{u}}_t = \int p(\mathbf{x}_{t+1}| \mathbf{x}_t, \mathbf{u}_t, \bar{\mathbf{u}}_t)\;p(\bar{\mathbf{u}}_t | \mathbf{x}_t, \mathbf{u}_t)\;d\bar{\mathbf{u}}_t \\
&= \int p(\mathbf{x}_{t+1}| \mathbf{x}_t, \mathbf{u}_t, \bar{\mathbf{u}}_t)\;p(\bar{\mathbf{u}}_t | \mathbf{u}_t)\;d\bar{\mathbf{u}}_t.
\end{align*}
Now in order to derive a variational lower-bound on the conditional log-likelihood $\log p(\mathbf{x}_{t+1}| \mathbf{x}_t, \mathbf{u}_t)$, we shall derive a variational lower-bound on the conditional log-likelihood $\log p(\mathbf{x}_{t+1}| \mathbf{x}_t, \mathbf{u}_t, \bar{\mathbf{u}}_t)$ as
\begin{align*}
\log p(\mathbf{x}_{t+1}| \mathbf{x}_t, \mathbf{u}_t, \bar{\mathbf{u}}_t) &\geq \mathbb{E}_{q^*}\big[\log p(\mathbf{x}_{t+1} | \mathbf{z}_t, \bar{\mathbf{z}}_t, \hat{\mathbf{z}}_{t+1}, \mathbf{x}_t, \mathbf{u}_t, \bar{\mathbf{u}}_t)\big] - \text{KL}\big(q^* \parallel p(\mathbf{z}_t, \bar{\mathbf{z}}_t, \hat{\mathbf{z}}_{t+1} | \mathbf{x}_t, \mathbf{u}_t, \bar{\mathbf{u}}_t)\big) \\
&= \mathbb{E}_{q^*}\big[\log p(\mathbf{x}_{t+1} | \mathbf{z}_t, \bar{\mathbf{z}}_t, \hat{\mathbf{z}}_{t+1}, \mathbf{x}_t, \mathbf{u}_t, \bar{\mathbf{u}}_t) + \log p(\mathbf{z}_t, \bar{\mathbf{z}}_t, \hat{\mathbf{z}}_{t+1} | \mathbf{x}_t, \mathbf{u}_t, \bar{\mathbf{u}}_t) - \log q^*\big] \\ 
&= \mathbb{E}_{q^*}\big[\log p(\mathbf{x}_{t+1}, \mathbf{z}_t, \bar{\mathbf{z}}_t, \hat{\mathbf{z}}_{t+1} | \mathbf{x}_t, \mathbf{u}_t, \bar{\mathbf{u}}_t) - \log q^*\big] \\
&\stackrel{\text{(a)}}{=} \mathbb{E}_{q^*}\big[\log p(\mathbf{x}_{t+1} | \hat{\mathbf{z}}_{t+1}) + \log p(\mathbf{z}_t | \mathbf{x}_t) + \log p(\bar{\mathbf{z}}_t | \mathbf{x}_t) + \log\delta(\hat{\mathbf{z}}_{t+1} | \mathbf{z}_t, \bar{\mathbf{z}}_t, \mathbf{u}_t, \bar{\mathbf{u}}_t) - \\ 
&\qquad\quad\;\; \log q_\phi(\hat{\mathbf{z}}_{t+1} | \mathbf{x}_{t+1}) - \log q_\varphi(\bar{\mathbf{z}}_t | \mathbf{x}_t, \hat{\mathbf{z}}_{t+1}) - \log\delta(\mathbf{z}_t | \bar{\mathbf{z}}_t, \hat{\mathbf{z}}_{t+1}, \mathbf{u}_t, \bar{\mathbf{u}}_t)\big] \\
&\stackrel{\text{(b)}}{=} \mathbb{E}_{q_\phi(\hat{\mathbf{z}}_{t+1} | \mathbf{x}_{t+1})}\big[\log p(\mathbf{x}_{t+1} | \hat{\mathbf{z}}_{t+1})\big] + \mathbb{E}_{\stackrel{q_\phi(\hat{\mathbf{z}}_{t+1} | \mathbf{x}_{t+1})}{q_\varphi(\bar{\mathbf{z}}_t | \mathbf{x}_t, \hat{\mathbf{z}}_{t+1})}}\big[\log p(\mathbf{z}_t | \mathbf{x}_t)\big] \\ 
&+ \mathbb{E}_{\stackrel{q_\phi(\hat{\mathbf{z}}_{t+1} | \mathbf{x}_{t+1})}{q_\varphi(\bar{\mathbf{z}}_t | \mathbf{x}_t, \hat{\mathbf{z}}_{t+1})}}\big[\log p(\bar{\mathbf{z}}_t | \mathbf{x}_t)\big] - \underbrace{\mathbb{E}_{q_\phi(\hat{\mathbf{z}}_{t+1} | \mathbf{x}_{t+1})}\big[\log q_\phi(\hat{\mathbf{z}}_{t+1} | \mathbf{x}_{t+1})\big]}_{H\big(q_\phi(\hat{\mathbf{z}}_{t+1} | \mathbf{x}_{t+1})\big)} \\ 
&- \mathbb{E}_{\stackrel{q_\phi(\hat{\mathbf{z}}_{t+1} | \mathbf{x}_{t+1})}{q_\varphi(\bar{\mathbf{z}}_t | \mathbf{x}_t, \hat{\mathbf{z}}_{t+1})}}\big[\log q_\varphi(\bar{\mathbf{z}}_t | \mathbf{x}_t, \hat{\mathbf{z}}_{t+1})\big] \\
&= \mathbb{E}_{q_\phi(\hat{\mathbf{z}}_{t+1} | \mathbf{x}_{t+1})}\big[\log p(\mathbf{x}_{t+1} | \hat{\mathbf{z}}_{t+1})\big] + \mathbb{E}_{\stackrel{q_\phi(\hat{\mathbf{z}}_{t+1} | \mathbf{x}_{t+1})}{q_\varphi(\bar{\mathbf{z}}_t | \mathbf{x}_t, \hat{\mathbf{z}}_{t+1})}}\big[\log p\big(\delta(\bar{\mathbf{z}}_t, \hat{\mathbf{z}}_{t+1}, \mathbf{u}_t, \bar{\mathbf{u}}_t)\big)\big] \\
&+ H\big(q_\phi(\hat{\mathbf{z}}_{t+1} | \mathbf{x}_{t+1})\big) - \mathbb{E}_{q_\phi(\hat{\mathbf{z}}_{t+1} | \mathbf{x}_{t+1})}\big[\text{KL}\big(q_\varphi(\bar{\mathbf{z}}_t | \mathbf{x}_t, \hat{\mathbf{z}}_{t+1}) \parallel p(\bar{\mathbf{z}}_t | \mathbf{x}_t)\big)\big] = \mathcal{L}_t^{\text{RCE}}.
\end{align*}
{\bf (a)} We replace $\log p(\mathbf{x}_{t+1}, \mathbf{z}_t, \bar{\mathbf{z}}_t, \hat{\mathbf{z}}_{t+1} | \mathbf{x}_t, \mathbf{u}_t, \bar{\mathbf{u}}_t)$ and $q^*$ using Equations~\ref{eq:gen-model} and~\ref{eq:q_fac}. 

{\bf (b)} The terms that contain $\delta(.|.)$ are zero.
\end{proof}

The terms in the variational lower-bound $\mathcal{L}_t^{\text{RCE}}$ can be written in closed form as

\begin{enumerate}
\item $\mathbb{E}_{q_{\phi}(\hat{\mathbf{z}}_{t+1}|\mathbf{x}_{t+1})} \big [\log p(\mathbf{x}_{t+1}|\hat{\mathbf{z}}_{t+1}) \big ]$

Using the reparameterization trick~\cite{vae2014}, we should first sample from $\mathcal{N}(\mathbf{\mu}_{\phi}(\mathbf{x}_{t+1}),\mathbf{\Sigma}_{\phi}(\mathbf{x}_{t+1}))$, i.e. we sample from a standard normal distribution $\mathbf{\epsilon} \sim \mathcal{N}(\mathbf{0},\mathbf{I})$ and transform it using $\mu_\phi(\mathbf{x}_{t+1})$ and $\mathbf{\Sigma}_\phi(\mathbf{x}_{t+1})$. When the covariance matrix $\mathbf{\Sigma}_\phi(\mathbf{x}_{t+1}) = \mathrm{diag}(\mathbf{\sigma}^2(\mathbf{x}_{t+1}))$ is diagonal, then the transformation is simply $\hat{\mathbf{z}}_{t+1} = \mathbf{\mu}_{\phi}(\mathbf{x}_{t+1})+\mathbf{\sigma}_{\phi}(\mathbf{x}_{t+1}) \odot \mathbf{\epsilon}$. Considering a Bernoulli distribution for the posterior of $\mathbf{x}_{t+1}$, the term inside the expectation is a binary cross entropy. 

\item $\mathbb{E}_{q_{\phi}(\hat{\mathbf{z}}_{t+1}|\mathbf{x}_{t+1})} \Big [ \text{KL} \big ( q_{\varphi}(\bar{\mathbf{z}}_t| \hat{\mathbf{z}}_{t+1},\mathbf{x}_t) \parallel p(\bar{\mathbf{z}}_t| \mathbf{x}_t)  \big ) \Big ]$

Similar to the previous term, to estimate the expected value we first need to sample from $\mathcal{N}(\mathbf{\mu}_{\phi}(\mathbf{x}_{t+1}),\mathbf{\Sigma}_{\phi}(\mathbf{x}_{t+1}))$, using the reparameterization trick. Note that $p(\bar{\mathbf{z}}_t| \mathbf{x}_t) = p(\mathbf{z}_t| \mathbf{x}_t)$ and $p(\mathbf{z}_t| \mathbf{x}_t)= q(\mathbf{z}_t| \mathbf{x}_t) = \mathcal{N}(\mathbf{\mu}_{\phi}(\mathbf{x}_t),\mathbf{\Sigma}_{\phi}(\mathbf{x}_t))$. For the $q_{\varphi}$ network, which is the transition network in our model, we have $q_{\varphi}(\bar{\mathbf{z}}_t| \hat{\mathbf{z}}_{t+1},\mathbf{x}_t) = \mathcal{N}(\mathbf{\mu}_{\varphi},\mathbf{\Sigma}_{\varphi})$. The KL term can be written as
\begin{align*}
\text{KL} \big(q_{\varphi}(\bar{\mathbf{z}}_t | \hat{\mathbf{z}}_{t+1},\mathbf{x}_t) \parallel p(\bar{\mathbf{z}}_t| \mathbf{x}_t)\big) = 
\frac{1}{2} \Big(\text{Tr}\big({\mathbf{\Sigma}_{\phi}(\mathbf{x}_t)}^{-1}\mathbf{\Sigma}_{\varphi}\big) &+ \big(\mu_{\phi}(\mathbf{x}_t) - \mu_{\varphi}\big)^{\top} {\mathbf{\Sigma}_{\phi}(\mathbf{x}_t)}^{-1} \big(\mu_{\phi}(\mathbf{x}_t) - \mu_{\varphi}\big) \\ 
&+ \log (\frac{|\mathbf{\Sigma}_{\phi}(\mathbf{x}_t)|}{|\mathbf{\Sigma}_{\varphi}|}) -n_z\Big). 
\end{align*}
\item $\text{H} \big ( q_{\phi}(\hat{\mathbf{z}}_{t+1}| \mathbf{x}_{t+1}) \big)$

The entropy term for the encoding network can be easily written in closed form as
\begin{equation}
\text{H} \big ( q_{\phi}(\hat{\mathbf{z}}_{t+1}| \mathbf{x}_{t+1}) \big )= \frac{1}{2} \log \big((2\pi e)^{n_z}|\mathbf{\Sigma}_{\phi}(\mathbf{x}_{t+1})| \big).
\end{equation}

\item $\mathbb{E}_{\substack{ q_{\phi}(\hat{\mathbf{z}}_{t+1}|\mathbf{x}_{t+1})\\q_{\varphi}(\bar{\mathbf{z}}_t|\mathbf{x}_t,\hat{\mathbf{z}}_{t+1}) } }\big [ \log p(\mathbf{z}_t| \mathbf{x}_t) \big ]$

Here we first need to sample from $\mathcal{N}(\mathbf{\mu}_{\phi}(\mathbf{x}_{t+1}),\mathbf{\Sigma}_{\phi}(\mathbf{x}_{t+1}))$ and $\mathcal{N}(\mathbf{\mu}_{\varphi},\mathbf{\Sigma}_{\varphi})$, using the reparameterization trick. Given that $p(\mathbf{z}_t | \mathbf{x}_t) = \mathcal{N}(\mathbf{\mu}_{\phi}(\mathbf{x}_t),\mathbf{\Sigma}_{\phi}(\mathbf{x}_t))$, the log term inside the expectation means that we want the output of transition network to be close to the mean of its distribution, up to some constant. 

\begin{equation}
\log p(\mathbf{z}_t| \mathbf{x}_t)= - \frac{1}{2} \Big ( \log \big((2\pi e)^{n_z} |\mathbf{\Sigma}_{\phi}(\mathbf{x}_t)| \big) +  (\mathbf{z}_t - \mu_{\phi}(\mathbf{x}_t))^{\top} {\mathbf{\Sigma}_{\phi}(\mathbf{x}_t)}^{-1} (\mathbf{z}_t - \mu_{\phi}(\mathbf{x}_t))\Big ).
\end{equation}
\end{enumerate}

\section{Implementation}
\label{sec:implement}

\textbf{Transition model structure:} $\mathbf{x}_t$ goes through one hidden layer with $\ell_1$ units and $\hat{\mathbf{z}}_{t+1}$ goes through one hidden layer with $\ell_2$ units. The outputs of the two hidden layers are concatenated and go through a network with two hidden layers of size $\ell_3$ and $\ell_4$, respectively, to build $\mu_{\varphi}$ and $\mathbf{\Sigma}_{\varphi}$.  $\bar{\mathbf{z}}_t$ is sampled from this distribution and is concatenated by the action. The result goes through a three-layer network with $\ell_5$, $\ell_6$, and $\ell_7$ units to build $\mathbf{M}_t$, $\mathbf{B}_t$, and $\mathbf{c}_t$. 

In the following we will specify the values for $\ell_i$'s for each of the four tasks used in our experiments.

\subsection{Planar system}

\textbf{Input: } $40 \times 40$ images  ($1600$ dimensions). 2-dimensional actions. $5000$ training samples of the form $(\mathbf{x}_t,\mathbf{u}_t,\mathbf{x}_{t+1})$

\textbf{Latent space: }2-dimensional 

\textbf{Encoder:} 3 Layers: 300 units- 300 units- 4 units (2 for mean and 2 for the variance of the Gaussian distribution)

\textbf{Decoder: } 3 Layers: 300 units- 300 units- 1600 units 

\textbf{Transition: } $\ell_1 = 100$- $\ell_2 = 5$- $\ell_3 = 100$- $\ell_4 = 4$- $\ell_5 = 20$- $\ell_6 = 20$- $\ell_7 = 10$

\textbf{Number of control actions: } or the planning horizon $T = 40$

\subsection{Inverted Pendulum}

\textbf{Input: } Two $48 \times 48$ images ($4608$ dimensions). 1-dimensional actions. $5000$ training samples of the form $(\mathbf{x}_t,\mathbf{u}_t,\mathbf{x}_{t+1})$

\textbf{Latent space: }3-dimensional 

\textbf{Encoder:} 3 Layers: 500 units- 500 units- 6 units (3 for mean and 3 for the variance of the Gaussian distribution)

\textbf{Decoder: } 3 Layers: 500 units- 500 units- 4608 units 

\textbf{Transition: } $\ell_1 = 200$- $\ell_2 = 10$- $\ell_3 = 200$- $\ell_4 = 6$- $\ell_5 = 30$- $\ell_6 = 30$- $\ell_7 = 12$

\textbf{Number of control actions: } or the planning horizon $T = 100$


\subsection{Cart-pole Balancing}

\textbf{Input: } Two $80 \times 80$ images ($12800$ dimensions). 1-dimensional actions. $15000$ training samples of the form $(\mathbf{x}_t,\mathbf{u}_t,\mathbf{x}_{t+1})$

\textbf{Latent space: } 8-dimensional 

\textbf{Encoder:} 6 Layers: convolutional layer: $32 \times 5 \times 5$; stride (1,1) - convolutional layer: $32 \times 5 \times 5$; stride (2,2) - convolutional layer: $32 \times 5 \times 5$; stride (2,2)  -convolutional layer: $10 \times 5 \times 5$; stride (2,2) - 200 units- 16 units (8 for mean and 8 for the variance of the Gaussian distribution)

\textbf{Decoder: } 6 Layers: 200 units- 1000 units- convolutional layer: $32 \times 5 \times 5$; stride (1,1)- Upsampling (2,2)- convolutional layer: $32 \times 5 \times 5$; stride (1,1)- Upsampling (2,2)- convolutional layer: $32 \times 5 \times 5$; stride (1,1)- Upsampling (2,2)- convolutional layer: $2 \times 5 \times 5$; stride (1,1) 

\textbf{Transition: } $\ell_1 = 300$- $\ell_2 = 10$- $\ell_3 = 300$- $\ell_4 = 16$- $\ell_5 = 40$- $\ell_6 = 40$- $\ell_7 = 32$

\textbf{Number of control actions: } or the planning horizon $T = 100$


\subsection{Three-Link Robot Arm}

\textbf{Input: } Two $128\times 128$ images ($32768$ dimensions). 3-dimensional actions. $30000$ training samples of the form $(\mathbf{x}_t,\mathbf{u}_t,\mathbf{x}_{t+1})$

\textbf{Latent space: }8-dimensional 

\textbf{Encoder:} 6 Layers: convolutional layer: $64 \times 5 \times 5$; stride (1,1) - convolutional layer: $32 \times 5 \times 5$; stride (2,2) - convolutional layer: $32 \times 5 \times 5$; stride (2,2)  -convolutional layer: $10 \times 5 \times 5$; stride (2,2) - 500 units- 16 units (8 for mean and 8 for the variance of the Gaussian distribution)

\textbf{Decoder: } 6 Layers: 500 units- 2560 units- convolutional layer: $32 \times 5 \times 5$; stride (1,1)- Upsampling (2,2)- convolutional layer: $32 \times 5 \times 5$; stride (1,1)- Upsampling (2,2)- convolutional layer: $32 \times 5 \times 5$; stride (1,1)- Upsampling (2,2)- convolutional layer: $2 \times 5 \times 5$; stride (1,1) 

\textbf{Transition: } $\ell_1 = 400$- $\ell_2 = 10$- $\ell_3 = 400$- $\ell_4 = 6$- $\ell_5 = 40$- $\ell_6 = 40$- $\ell_7 = 48$

\textbf{Number of control actions: } or the planning horizon $T = 100$

\section{E2C Graphical Model}
\label{sec:appendix-E2C}

Since the original E2C paper does not provide a graphical model for its generative and recognition models, in this section, we present a graphical model that faithfully corresponds to the lower-bound reported in Equation~12 of the E2C paper~\cite{watter2015embed}. 

At high-level, the generative model involves two latent variables $\mathbf{z}_t$ and $\hat{\mathbf{z}}_{t+1}$, with the joint factorization (note that we omit the dependency on $\mathbf{u}_t$ for brevity)
\begin{equation*}
p(\mathbf{x}_t,\mathbf{x}_{t+1},\mathbf{z}_t, \hat{\mathbf{z}}_{t+1}) =
p(\mathbf{x}_{t+1}|\hat{\mathbf{z}}_{t+1})\;p(\mathbf{x}_t|\mathbf{z}_t)\;p(\hat{\mathbf{z}}_{t+1}|\mathbf{z}_t,\mathbf{x}_t)\;p(\mathbf{z}_t).
\end{equation*}
With the above generative model, any recognition model of the form (note that we borrow the generative transition dynamic $p(\hat{\mathbf{z}}_{t+1}|\mathbf{z}_t,\mathbf{x}_t)$)
\begin{equation*}
q(\mathbf{z}_t,\hat{\mathbf{z}}_t|\mathbf{x}_{t}, \mathbf{x}_{t+1})=
q(\mathbf{z}_t|\mathbf{x}_t)\;p(\hat{\mathbf{z}}_{t+1}|\mathbf{z}_t,\mathbf{x}_t)
\end{equation*}
gives rise to the following variational lower-bound of the log-pair-marginal 
\begin{align} 
\log\; p(\mathbf{x}_t, \mathbf{x}_{t+1}) &\ge
\mathbb{E}_{q(\mathbf{z}_t,\hat{\mathbf{z}}_{t+1}|\mathbf{x}_t, \mathbf{x}_{t+1})}
\left \{ \log\frac{p(\mathbf{x}_t,\mathbf{x}_{t+1},\mathbf{z}_t,\hat{\mathbf{z}}_{t+1})}{q(\mathbf{z}_t,\hat{\mathbf{z}}_{t+1}|\mathbf{x}_t, \mathbf{x}_{t+1})}\right\} \nonumber \\
&= \mathbb{E}_{q(\mathbf{z}_t|\mathbf{x}_t)} \big[\log p(\mathbf{x}_t|\mathbf{z}_t)\big] + 
\mathbb{E}_{
q(\hat{\mathbf{z}}_{t+1}|\mathbf{x}_{t})} \big[\log p(\mathbf{x}_{t+1}|\hat{\mathbf{z}}_{t+1})\big] - \text{KL}\big(q(\mathbf{z}_t|x_t)\Vert p(\mathbf{z}_t)\big). \label{eq:e2c-rederived-bound}
\end{align}
Note that the form of Equation~\ref{eq:e2c-rederived-bound} above is equivalent to the bound in Equation~12 in~\cite{watter2015embed}. The E2C objective (their Equation~11) includes another auxiliary KL term to maintain the consistency of the embedding as it evolves over time. This term is not needed in our RCE model.

Next, we give our interpretation of Equations~8 and~10 in~\cite{watter2015embed}. We claim that E2C works with the following transition dynamics 
\begin{equation*}
q(\hat{\mathbf{z}}_{t+1}|\mathbf{z}_t,\mathbf{x}_t) = p(\hat{\mathbf{z}}_{t+1}|\mathbf{z}_t,\mathbf{x}_t) = \int_{\bar{\mathbf{z}}_t} p(\hat{\mathbf{z}}_{t+1}|\bar{\mathbf{z}}_t,\mathbf{z}_t)p(\bar{\mathbf{z}}_t|\mathbf{x}_t)
\end{equation*}
where $\bar{\mathbf{z}}_t$ plays the role of the linearization point in the LQR model and $p(\hat{\mathbf{z}}_{t+1}|\bar{\mathbf{z}}_t,\mathbf{z}_t)$ is deterministic (an added Gaussian noise can also be handled in a straightforward manner)
\begin{equation*}
\hat{\mathbf{z}}_{t+1} = \mathbf{A}(\bar{\mathbf{z}}_t)\mathbf{z}_t + \mathbf{B}(\bar{\mathbf{z}}_t)\mathbf{u}_t + \mathbf{o}(\bar{\mathbf{z}}_t).
\end{equation*}
Furthermore, the recognition model has an additional constraint $q(\bar{\mathbf{z}}_t|\mathbf{x}_t) = p(\bar{\mathbf{z}}_t|\mathbf{x}_t) = q(\mathbf{z}_t | \mathbf{x}_t)$.

Under these conditions, the implementation of the lower-bound will give rise to exactly their Equations~8 and~10 (minus some typos). We note that there is a typo in their Equation~10: the matrices and offset of the transition dynamics should be functions of the linearization point $\bar{\mathbf{z}}_t$. The first two lines in Equation~8 describe the sampling of $q(\hat{\mathbf{z}}_{t+1}|\mathbf{x}_t)$: the first line should read as the sampling of the auxiliary variable $\bar{\mathbf{z}}_t$. The second line is the sampling of $\hat{\mathbf{z}}_{t+1}$, where the matrices and offset $\mathbf{A},\mathbf{B},\mathbf{o}$ are functions of $\bar{\mathbf{z}}_t$, sampled in the first line. The second line holds due to the fact that given $\bar{\mathbf{z}}_t$, $\mathbf{z}_{t+1}$ has a linear dynamics with known coefficients, and $\mathbf{z}_t|\mathbf{x}_t$ is Gaussian $\mathcal{N}(\mu_t,\Sigma_t)$ under $q$ and hence can be marginalized out.

\vspace{1cm}
\begin{figure}[!h]
    \centering
    \includegraphics[trim = 0mm 7mm 0mm 10mm,height=5cm]{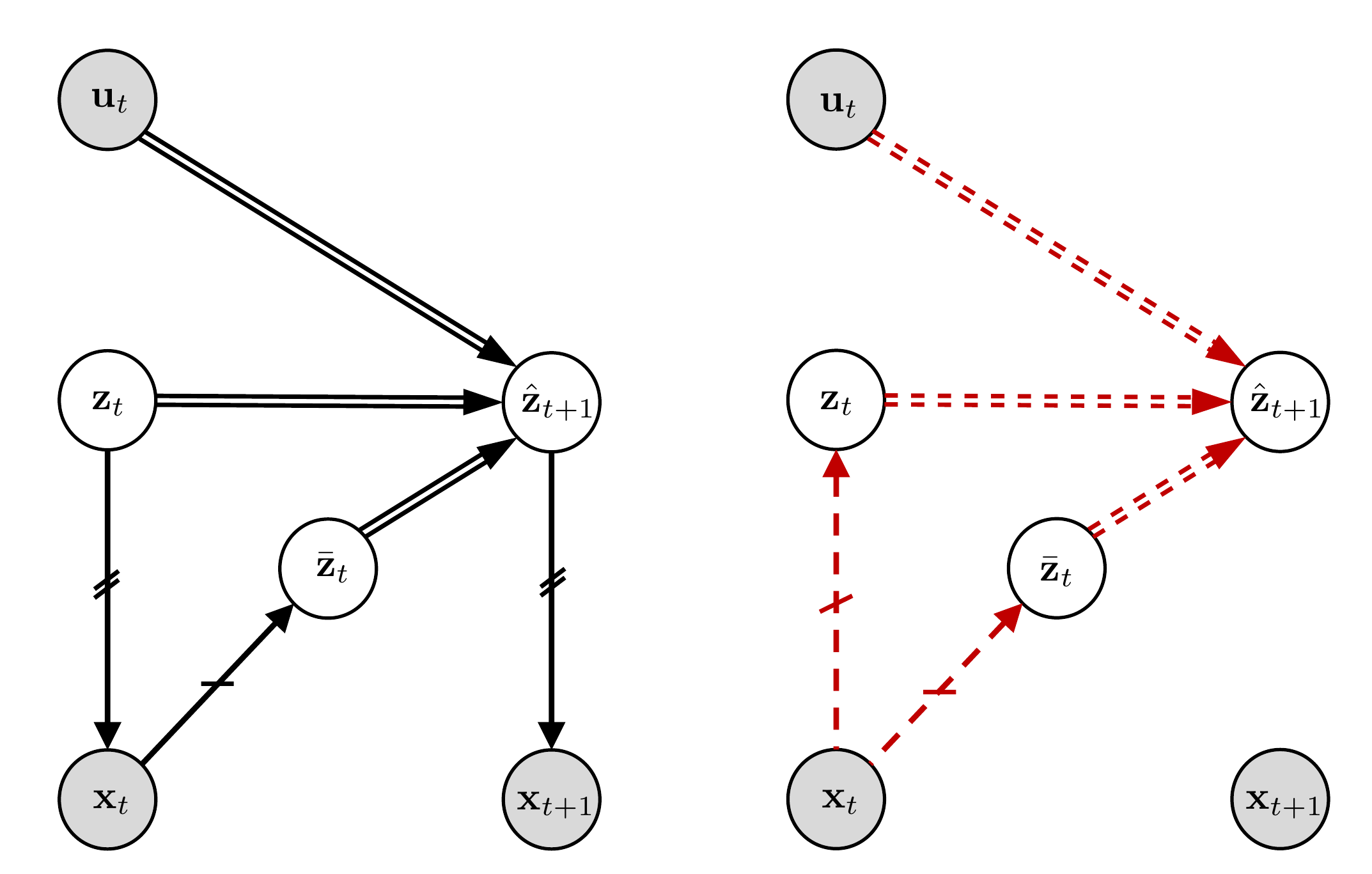} 
    \caption{E2C Graphical Model- Left: generative model ($p$) and right: recognition model ($q$). Note the tying between the dynamics in $p$ and $q$, i.e. $q(\hat{\mathbf{z}}_{t+1}|\mathbf{z}_t,\mathbf{x}_t)=p(\hat{\mathbf{z}}_{t+1}|\mathbf{z}_t,\mathbf{x}_t)$. Also, note the tying of the decoder parameters $p(\mathbf{x}_t| \mathbf{z}_t)$ and $p(\mathbf{x}_{t+1}|\hat{\mathbf{z}}_{t+1})$, which is shown by the hatch marks. The parameter of the networks for $p(\bar{\mathbf{z}}_t|\mathbf{x}_t)$, $q(\bar{\mathbf{z}}_t|\mathbf{x}_t)$, and $q(\mathbf{z}_t|\mathbf{x}_t)$ are also tied, marked by the dashes on this figure.} \label{fig:e2c_models}%
\end{figure}

\end{document}